\documentclass[letterpaper, 10 pt, conference]{ieeeconf}  

\IEEEoverridecommandlockouts                              
\usepackage{graphicx}          
\usepackage{graphics} 
\usepackage{amsmath} 
\usepackage{amssymb}  
\usepackage{mathtools}
 
\usepackage{amsthm}
\usepackage{color}
\usepackage{cite}
\usepackage[ruled]{algorithm2e} 
\usepackage[export]{adjustbox}

\newtheorem{thm}{Theorem}

\newtheorem{definition}{Definition}

\SetKwInOut{Initialize}{Initialization}

\title{\LARGE \bf
Trajectory Generation for Robotic Systems with Contact Force Constraints
}

\author{Jaemin Lee$^{1}$, Efstathios Bakolas$^{2}$, and Luis Sentis$^{3*}$
\thanks{$^{1}$J. Lee is with the Department of Mechanical Engineering and Human-centered Robotics Laboratory, The University of Texas at Austin, Austin, TX, 78712, USA
	{\tt\small jmlee87@utexas.edu}}
\thanks{$^{2}$E. Bakolas is with Faculty of the Department of Aerospace Engineering and Engineering Mechanics, The University of Texas at Austin, Austin, TX, 78712, USA
    {\tt\small bakolas@austin.utexas.edu}}%
\thanks{$^{3}$L. Sentis is with  Faculty of the Department of Aerospace Engineering and Engineering Mechanics, The University of Texas at Austin, Austin, TX, 78712, USA. *L. Sentis is the corresponding author. {\tt\small lsentis@austin.utexas.edu}}
}

\begin{document}

\maketitle
\thispagestyle{empty}
\pagestyle{empty}

\begin{abstract}
This paper presents a trajectory generation method for contact-constrained robotic systems such as manipulators and legged robots. Contact-constrained systems are affected by the interaction forces between the robot and the environment. In turn, these forces determine and constrain state reachability of the robot parts or end effectors. Our study subdivides the trajectory generation problem and the supporting reachability analysis into tractable subproblems consisting of a sampling problem, a convex optimization problem, and a nonlinear programming problem. Our method leads to significant reduction of computational cost. The proposed approach is validated using a realistic simulated contact-constrained robotic system. 
\end{abstract}

\section{INTRODUCTION}
We aim to control contact-constrained robotic systems using optimal control in a computationally feasible way. In the case of high dimensional systems such as legged humanoid robots, a widely used control method is to create desired trajectories without prior knowledge of their feasibility, then rely on a feedback tracking controller to instantaneously realize them by projecting the controls into the null space of the constraints 
\cite{sentis2005synthesis,mistry2012operational,escande2014hierarchical}. The drawback of this method is that the desired trajectories are often infeasible forcing the feedback controllers to fulfill them only in a least square error sense. Other methods used in humanoid robots rely on using simplified models to design feasible trajectories, e.g. using center of mass dynamics subject to contact constraints \cite{stephens2010dynamic}. However, those methods cannot guarantee the feasibility of the trajectories because the mechanical degrees-of-freedom of the robot are ignored. In contrast, our approach generates feasible trajectories that fully comply with the robot's mechanics and dynamics as well as its contact state with its environment. In order to obtain feasible trajectories for the desired goals, we broadly employ reachability analysis. In particular contact constraints need different treatment than state constraints since they need to fulfill cone formulation requirements. This type of formulation has not been employed before for optimal control of robotic systems.

Reachability analysis is often used in optimal control for analyzing the performance and safety of various types of dynamical systems with bounded uncertainty, \cite{jia2002min,bravo2006robust,gonzalez2011online,rakovic2012parameterized,subramanian2017novel}, for hybrid dynamical systems \cite{mitchell2001validating,tomlin2000game,habets2006reachability,mitchell2005time,maiga2016comprehensive}, and for stochastic systems \cite{summers2013stochastic,lesser2014reachability}. A common method to perform reachability analysis in the continuous domain is by solving the Hamilton-Jacobi-Bellman PDE \cite{kurzhanski2001dynamic,asarin2000approximate,kariotoglou2013approximate}. But methods based on this process result in exponentially increasing computational cost as a function of the system's state and the discretization step. Therefore, for robotics it is not possible to use Hamilton-Jacobi-Bellan methods due to these limitations on scalability. Another approach is to employ the logarithmic norm of a Jacobian matrix producing over-approximated bounds of the reachable sets \cite{maidens2015reachability} and checking feasibility via simulations \cite{arcak2017simulation}. 
However, this kind of method does not incorporate system constraints as we do. 

In robotics, configuration-based reachability analysis has been broadly used for motion planning of complex robotic systems. For instance, collision-free reachability maps in the configuration space are employed for the planning of humanoid motions \cite{yang2017efficient} but without addressing contact forces or robot dynamics. The Monte-Carlo method has been used to obtain piecewise end-effector position samples by exploring configuration space samples and center of mass positions \cite{guan2006reachable}, however they ignore the robot's dynamics. Another idea is to connect nodes obtained via Rapidly-Exploring Random Trees using reachability analysis \cite{shkolnik2009reachability}. However, these two last works have not been extended to dynamical systems with contact force constraints. 


Fundamentally, our method generates trajectories for contact force constrained robotic systems with differential constraints. We rely on randomly generated samples \cite{boor1999gaussian,patil2012estimating,carpentier2017multi,hauser2008motion} and null-space projections associated with the differential constraints \cite{stilman2007task}. This approach is more efficient than employing Monte-Carlo methods to generate the set of samples fulfilling the system constraints. A key novelty in our sampling process is using a convex optimization stage to determine whether samples fulfill contact force constraints. We define the fraction of reachable samples (FRS) with respect to samples falling in the output regions that are feasible. The FRS, which has a dependency with the number of samples, is used as an indicator to guide the sampling process. We first solve a DP problem to obtain a candidate trajectory of the output guided by the properties of the samples. After this process we perform an additional optimal control procedure based on the reachability between the sampled points. We utilize a non-convex hull, which is the envelope of a set containing the output samples, to describe the set of reachable samples propagated from a given initial state over a finite time interval. By using small time intervals for generating trajectories between adjacent regions we make use of model approximations that significantly reduce the computational burden required for state propagation. In addition, we increase computational efficiency by propagating only the dynamics of boundary states.


This paper is organized as follows. We introduce notations, the state space model of the constrained robotic system, and its time-discretization in Section \ref{sec2}. In Section \ref{sec3}, we explain how to obtain the samples that fulfill all system constraints and partition the sampled space based on the fraction of feasible samples. In Section \ref{sec4}, we characterize practical reachable sets and use them to design optimal control problems to generate trajectories over short time intervals. An example problem and simulation results are shown  in Section \ref{sec5}.

\section{PRELIMINARIES}
\label{sec2}
\subsection{Notations}
The sets of real $n$-dimensional vectors and $m\times n$ matrices are denoted by $\mathbb{R}^{n}$ and $\mathbb{R}^{m\times n}$, respectively. $\mathbb{R}^{+}$ and $\mathbb{R}^{++}$ indicate the sets of non-negative and strictly positive real numbers. $\mathbb{Z}^{+}$ and $\mathbb{Z}^{++}$ represent the sets of non-negative and strictly positive integers. When considering $z_1, z_2 \in \mathbb{Z}^{+}$ where $z_{1}\leq z_2$, the interval of integers between $z_1$ and $z_2$ is represented as $\left[z_1, z_2 \right]_d$ where $d$ stands for discretization. The space of real symmetric $n\times n$ matrices is denoted by $\mathbb{S}_{n}$ and the spaces of positive semidefinite and positive definite matrices are denoted by $\mathbb{S}_{n}^{+}$ and $\mathbb{S}_{n}^{++}$, respectively. Given a matrix $A$, $A^{\dag}$ and $\ker(A)$ denote the Moore-Penrose pseudo inverse and the kernel of $A$, respectively. Given multiple matrices $A_{1}, \dots, A_{k}$ or a set $\mathcal{A} = \left\{A_{1}, \dots, A_{k} \right\}$, $\mathrm{Vertcat}\left(A_{1}, \dots, A_{k} \right)$ or $\mathrm{Vertcat}\left(\mathcal{A}\right)$ denote the matrix formed by vertically concatenating the matrices $A_{1}$ to $A_{k}$. A diagonal matrix in $\mathbb{R}^{k \times k}$ with diagonal components $a_{1}, \cdots, a_{k}$ is denoted by $\mathrm{diag}\left( a_{1}, \cdots, a_{k} \right)$. Considering a vector $a\in \mathbb{R}^{n}$,  $\|a\|$ and $\|a\|_{\infty}$ denote the $2$-norm and $\infty$-norm of the vector $a$, respectively. $\mathbb{E}[.]$ represents the probabilistic expectation operator. Given a set $\mathcal{A} \subseteq \mathbb{R}^{n}$, $\mathrm{Int}(\mathcal{A})$ and $\mathrm{Ext}(\mathcal{A})$ denote the interior and the exterior of $\mathcal{A}$. Furthermore, $\mathrm{card}\left(\mathcal{A}\right)$, $\mathrm{bd}\left( \mathcal{A} \right)$, and $\mathrm{Nconv}\left(\mathcal{A}\right)$ represent the cardinality, the boundary, and the non-convex hull of the set $\mathcal{A}$, respectively. Given two sets $\mathcal{A}_{1}$ and $\mathcal{A}_{2}$, the relative complement of $\mathcal{A}_{1}$ with respect to $\mathcal{A}_{2}$ is denoted by $\mathcal{A}_{1} \backslash \mathcal{A}_2$, that is, $\mathcal{A}_{1} \backslash \mathcal{A}_{2} \coloneqq \left\{x\in \mathcal{A}_1: x \notin \mathcal{A}_{2} \right\}$. When $\mathcal{A} \subsetneq \mathbb{R}$, $\max\left( \mathcal{A} \right)$ and $\min \left( \mathcal{A} \right)$ denote the maximum and the minimum values among the elements of the set $\mathcal{A}$. Finally, if the $k$-th derivative of the function $f$ exists and is continuous, the function is said to be of class $\mathcal{C}^{k}$.  

\subsection{State Space Model of Robotic System}
The equation of motion of a multi-body dynamical system enduring contact forces with the environments can be described as follows:
\begin{equation}
M(q) \ddot{q} + b(q,\dot{q}) + p(q) = u + J_{c}^{\top} F_{c}
\end{equation}
where $q\in \mathbb{R}^{n_q}$, $u \in \mathbb{R}^{n_u}$, $M(q)\in\mathbb{S}_{n_{q}}^{++}$, $b(q,\dot{q}) \in \mathbb{R}^{n_q}$, $p(q) \in \mathbb{R}^{n_q}$, $F_{c} \in \mathbb{R}^{n_c}$, and $J_{c} \in \mathbb{R}^{n_c \times n_q}$ are the joint variables, input commands, mass/inertia matrix, Coriolis/centrifugal force, gravitational force, contact force, and the Jacobian matrix corresponding to the contact force constraint, respectively. We can transform the above equation into state space form as
\begin{equation}
\begin{split}
	&\dot{x} = f_{x}(x) + f_u(x)u + f_{c}(x)F_{c} \\
    &f_{x}(x) = \left[\begin{array}{c} x_{2} \\ M^{-1}(x_1)\left(-b(x_1, x_2) -p(x_1) \right) \end{array} \right] \\
    &f_{u}(x) = \left[\begin{array}{c} 0 \\ M^{-1}(x_{1}) \end{array} \right], \quad f_{c}(x) = \left[ \begin{array}{c} 0 \\ J_{c}^{\top}(x_1)  \end{array}\right]
\end{split}
\label{eq:state_eq}
\end{equation}
where $x = \left[\begin{array}{cc} x_1^{\top}& x_{2}^{\top}\end{array}\right]^{\top} \in \mathbb{R}^{n_x}$, $x_1 = q$, and $x_{2}=\dot{q}$. $f_x: \mathbb{R}^{n_x} \mapsto \mathbb{R}^{n_x}$, $f_{u} :\mathbb{R}^{n_x} \mapsto \mathbb{R}^{n_x \times n_u}$, and $f_{c}: \mathbb{R}^{n_x} \mapsto \mathbb{R}^{n_x\times n_{c}}$. The joint position and velocity limits of the robot are described as the state constraints, $C_x \left(x\right) \leq 0$, and torque limits are described as input constraints, $C_{u} \left( u \right) \leq 0$. In additional, more complicated interactions such as contact wrench cones constraints, are described as mixed state-input constraints $C_{x,u}\left(x,u\right) \leq 0$. Here, the constraint functions are $C_x: \mathbb{R}^{n_x} \mapsto \mathbb{R}^{n_{C_x}}$, $C_u: \mathbb{R}^{n_u} \mapsto \mathbb{R}^{n_{C_u}}$, and $C_{x,u}: \mathbb{R}^{n_x + n_u} \mapsto \mathbb{R}^{n_{C_{xu}}}$, which are $\mathcal{C}^{1}$. We discretize the state space dynamics in (\ref{eq:state_eq}) as:
\begin{equation}
\begin{split}
    x_{k+1} &= x_{k} + {B}_{1}(x_k, u_k, F_{c,k}) \Delta t  \\
     &\quad + \frac{{B}_{2}(x_k, u_k, F_{c,k}) \left( \Delta t \right)^{2}}{2} + \mathcal{O}\left(\left( \Delta t \right)^{2} \right) \\
    &= \mathcal{F} \left(x_k, u_{k}, F_{c,k} \right)
\end{split}
\label{eq:time_discrete}
\end{equation}
where $k \in \left[0, N-1\right]_{d}$. $\Delta t$ and $\mathcal{O}\left(\left(\Delta t \right)^{2} \right)$ denote the time discretization increment and terms higher than $2$nd order in Taylor series expansion, respectively. ${B}_{1}\left( x_{k}, u_{k}, F_{c,k} \right) \coloneqq f_{x}\left(x_{k}\right) + f_u\left( x_{k} \right) u_{k} + f_{c}\left( x_{k} \right) F_{c,k}$ and ${B}_{2}\left( x_{k}, u_{k}, F_{c,k} \right) \coloneqq \frac{\partial {B}_{1}\left(x_{k}, u_{k}, F_{c,k} \right) }{\partial x} {B}_{1}(x_{k},u_{k}, F_{c,k})$. 

The output state of the robotic system is defined as $y = g(x)$ where $y \in \mathbb{R}^{n_{y}}$ and $g:\mathbb{R}^{n_{x}} \mapsto \mathbb{R}^{n_{y}}$. $g$ is continuous and differentiable. Our problem concerns the generation of feasible state and input trajectories given desired output goals and initial system states maintaining the solid contact with respect to the nonlinear system model in (\ref{eq:time_discrete}).

\section{SAMPLING-BASED APPROACH}
\label{sec3}
As we mentioned earlier, sampling based methods enable to solve complex computational problems like ours. In this section, we introduce the way to obtain the samples fulfilling the given constraints.

\subsection{Mathematical Definitions for Sampling}
We draw random samples of the system's states from a Gaussian distribution $x \sim \mathcal{N} \left( \mu_x, \mathbf{\Sigma}_{x} \right)$ where $x\in \mathbb{R}^{n_x}$, $\mu_{x} \in \mathbb{R}^{n_x}$, and $\mathbf{\Sigma}_{x} \in \mathbb{S}_{n_x}^{++}$ denote the sampled state vector, its mean, and its covariance matrix, respectively, where $\mu_{x} \coloneqq \mathbb{E}\left[x\right]$ and $\mathbf{\Sigma}_{x} \coloneqq \mathbb{E}\left[ (x-\mu_{x})(x-\mu_{x})^{\top}\right]$. In robotics, we can define $\mu_{x}$ and $\mathbf{\Sigma}_x$ based on joint position and velocity limits. 

Given $n_{e}$ equality constraints, we describe them using the function $C_{e,[h_e]}(x) = 0$ where $h_e \in [1, n_e]_d$ is an index. This index is used to address multiple equality constraints separately. Likewise, given $n_{i}$ inequality constraints, we describe them using the function $C_{i,[h_i]}(x) \leq 0$  where $h_{i} \in [1, n_i]_d$ is also an index. $C_{e,[h_e]}: \mathbb{R}^{n_{x}} \mapsto \mathbb{R}^{n_{e,h_e}}$ and $C_{i,[h_i]}:\mathbb{R}^{n_x} \mapsto \mathbb{R}^{n_{i,h_i}}$ denote functions for the $h_e$-th equality and $h_i$-th inequality constraints, respectively, and both functions are $\mathcal{C}^{1}$ functions. Then, we define vectors of values for the equality and inequality constraint functions in terms of the state sample $x$.   
\begin{equation}
\begin{split}
\mathcal{V}_{E}(x) & \coloneqq \mathrm{Vertcat}\left(C_{e,[1]}(x), \cdots, C_{e,[n_e]}(x) \right) \\
\mathcal{V}_{I}(x) &\coloneqq \mathrm{Vertcat}\left(C_{i,[1]}(x), \cdots, C_{e,[n_i]}(x) \right)
\end{split}
\end{equation}
where $\mathcal{V}_{E}(x) \in \mathbb{R}^{\sum_{h_e=1}^{n_e}n_{h_e}}$ and $\mathcal{V}_{I}(x) \in \mathbb{R}^{\sum_{h_i=1}^{n_i}n_{h_i}}$. For dividing the indices of the constraint functions, we define two sets as follows:
\begin{equation}
\begin{split}
	H_{e}(x) &\coloneqq \left\{ h_e \in \mathbb{Z}^{++}:  \left\| C_{e,[h_e]}(x)  \right\| \leq \varepsilon_{h},  h_e\in \left[1,n_{e} \right]_{d} \right\} \\
    H_{i}(x) &\coloneqq \left\{ h_i \in \mathbb{Z}^{++}:  C_{i,[h_i]} (x) \leq 0,  h_i\in \left[1,n_{i} \right]_{d} \right\} \textrm{.}
\end{split}
\label{index}
\end{equation}
In addition, $H_{\backslash e}(x) \coloneqq [1,n_e]_d \backslash H_{e}(x)$ and $H_{\backslash i}(x) \coloneqq [1, n_{i}]_d \backslash H_{i}(x)$, respectively. To split all constraints into feasible and infeasible constraints in terms of the random sample $x$, we define the vectors whose elements are function values with respect to the index sets defined in (\ref{index}) as follows: 
\begin{equation}
\begin{split}
&\mathcal{V}_{e}(x) \coloneqq \mathrm{Vertcat} \left( C_{e,[h]}(x):\forall h \in H_{e}(x) \right) \\
&\mathcal{V}_{\backslash e}(x) \coloneqq \mathrm{Vertcat} \left( C_{e,[h]}(x):\forall h \in H_{\backslash e}(x) \right)  \\
&\mathcal{V}_{i}(x) \coloneqq \mathrm{Vertcat} \left(C_{i,[h]}(x): \forall h \in H_{i}(x) \right) \\ 
&\mathcal{V}_{\backslash i}(x) \coloneqq \mathrm{Vertcat} \left(C_{i,[h]}(x): \forall h \in H_{\backslash i}(x) \right)\textrm{.}
\end{split}
\label{feasible_sets}
\end{equation}
Since all constraint functions are differentiable, we can compute the Jacobian matrices of the constraint functions such that $J_{C_e,[h]}(x)\coloneqq \frac{\partial C_{e,[h]}}{\partial x}(x) \in \mathbb{R}^{n_{e,h}\times n_x}$ and $J_{C_i,[h]}(x)\coloneqq \frac{\partial C_{i,[h]}}{\partial x}(x) \in \mathbb{R}^{n_{i,h} \times n_x}$. In addition, we can define matrices by vertically concatenating the Jacobian matrices of the constraint functions with respect to the categorized index in (\ref{index}):   
\begin{equation}
\begin{split}
&\mathcal{J}_{e}(x) \coloneqq \mathrm{Vertcat}\left(J_{C_e,[h]}(x): \forall h \in H_{e}(x)  \right)\\
&\mathcal{J}_{\backslash e}(x) \coloneqq \mathrm{Vertcat}\left(J_{C_e,[h]}(x): \forall h \in H_{\backslash e}(x)  \right) \\
	&\mathcal{J}_{i}(x) \coloneqq \mathrm{Vertcat}\left(J_{C_i,[h]}(x): \forall h \in H_{i}(x)  \right) \\
    &\mathcal{J}_{\backslash i}(x) \coloneqq \mathrm{Vertcat}\left(J_{C_i,[h]}(x): \forall h \in H_{\backslash i}(x)  \right)
\end{split}    
\end{equation}
where all Jacobian matrices $\mathcal{J}_{e}(x)$, $\mathcal{J}_{\backslash e}(x)$, $\mathcal{J}_{i}(x)$, and $\mathcal{J}_{\backslash i}(x)$ are assumed as full row rank matrices. In order to have all sample states satisfying the state constraints, we will solve a least square error problem using the Moore-Penrose pseudo inverse.

\subsection{Update of Samples for the State Constraints}
We update the state sample for fulfilling all state constraints in the least square error sense. In addition, the orthogonal projection onto the kernel space of the constraint function is utilized to prevent modifications of the function values after fulfilling the constraints. Given a state sample $x^{l} \in \mathcal{X}$ in the $l$-th update iteration, we can compute the vertically concatenated Jacobian matrices $\mathcal{J}_{e}(x^{l})$ and $\mathcal{J}_{\backslash e}(x^{l})$. The orthogonal projection onto $\ker\left(\mathcal{J}_{e}(x)\right)$ is defined as
\begin{equation}
	P_{e}(x^{l}) \coloneqq  I_{n_x} -  \mathcal{J}_{e}^{\dag}(x^{l}) \mathcal{J}_{e}(x^{l})
    \label{eq:null_projection}
\end{equation}
where $I_{n_x}$ is the $n_x \times n_x$ identity matrix and $P_{e}: \mathbb{R}^{n_x} \mapsto \mathbb{R}^{n_x \times n_x}$. Based on the Jacobian matrix $\mathcal{J}_{\backslash e}(x^{l})$ and $P_{e}(x^{l})$, the sampled state is updated as follows:
\begin{equation}
	x^{l+1} = x^{l} - \alpha P_{e}(x^l)\left(\mathcal{J}_{\backslash e}\left(x^{l}\right) P_{e}\left( x^l \right) \right)^{\dag} \mathcal{V}_{\backslash e}(x^l)
    \label{state_update_1}
\end{equation}
where $l \in \left[0, N_{\mathrm{iter}}\right]_{d}$ and the initial state is the originally sampled state, that is, $x^0 = x$. In addition, $\alpha \in \mathbb{R}^{++}$ is the gain for regulating the convergence speed to exponentially reduce $\|\mathcal{V}_{\backslash e}\left( x^{l} \right)\|$ to $0$. When $\| \mathcal{V}_{\backslash e}\left( x^{l} \right) \| \leq \varepsilon_{e}$, the iteration process ends and the state sample $x$ is altered to the result $x^{l}$ in the set $\mathcal{X}$. If $\| \mathcal{V}_{\backslash e}\left( x^{N_{\mathrm{iter}}} \right)\| > \varepsilon_{e}$, the state sample is discarded from the set $\mathcal{X}$ for the computational efficiency of the method. 

After updating all samples in $\mathcal{X}$, we do not want to update the function values for which the constraints are already fulfilled. Therefore, an augmented Jacobian is defined with respect to the state $x^{l}$ as  
\begin{equation}
	\mathcal{J}_{aug}(x^{l}) \coloneqq \mathrm{Vertcat} \left( \mathcal{J}_{e}(x^{l}), \mathcal{J}_{i}(x^{l}) \right)
\end{equation}
where $P_{aug}(x^{l})$ is computed in the same manner in ($\ref{eq:null_projection}$). Using the Jacobian $\mathcal{J}_{\backslash i}(x^{l})$ projected onto $\mathrm{ker}\left(\mathcal{J}_{aug}(x^{l})\right)$, we can update the state sample without any modification of function values for already fulfilled constraints. The update of state sample is 
\begin{equation}
\begin{split}
	&x^{l+1} = x^{l} + \alpha P_{aug}(x^{l}) \left( \mathcal{J}_{\backslash i}(x^{l}) P_{aug}(x^{l}) \right)^{\dag} \mathcal{E}_{i}(x^{l})\\
	&\mathcal{E}_{i}(x^{l}) \coloneqq \mathcal{V}_{\backslash i}^{int}(x^{l}) - \mathcal{V}_{\backslash i}(x^{l})
\end{split}    
\end{equation} 
where $\mathcal{V}_{\backslash i}^{int}(x^{l})$ denotes the vector vertically concatenating the interior points fulfilling the constraints $C_{i,[h]}(x^{l})\leq0$ for all $h \in H_{\backslash i}(x^{l})$ . This update is terminated when the state $x^{l}$ fulfills the inequality constraints and the existing component $x$ is replaced by $x^{l}$. If the state update cannot be satisfied within $N_{\mathrm{iter}}$ iterations, the state sample is discarded from $\mathcal{X}$. Therefore, all components in the state set $\mathcal{X}$ fulfill the state constraints.        

\subsection{Sample Evaluation Given Contact Force Constraints}
\label{sec3_C}
In this section, we check both input and contact force constraints in terms of the samples in $\mathcal{X}$. Optimization techniques are broadly utilized to find the contact force for the mechanical systems, thus, we also formulate the optimization problem with quadratic cost function to obtain the contact force in terms of the samples as follows: 
\begin{subequations}
\begin{align}
	\min_{F_{c}} &\quad F_{c}^{\top} W_{c} F_{c} \\
    \textrm{s.t.} &\quad  x_{k+1} = \mathcal{F}\left(x_{k},u,F_{c}\right),  \quad x_{k}, x_{k+1} \in \mathcal{X},\\
    & \quad \mathcal{D}_{c}(x_{k})F_{c} \leq 0, C_{u} (u) \leq 0, C_{x,u}(x_{k},u) \leq 0
\end{align}
\label{eq:opt1}
\end{subequations}
where $\mathcal{D}_{c}: \mathbb{R}^{n_{x}} \mapsto \mathbb{R}^{n_{c'}\times n_{c}}$ denotes the unilateral constraints using a polyhedral approximation of the friction cone \cite{caron2015stability}. Solving this optimization problem for all states in $\mathcal{X}$, the set of feasible states in discrete state space $\mathcal{X}_{R}$ can be defined as the collection of the samples which result in the optimal contact force and input command given (\ref{eq:opt1}).

\subsection{Fraction of Reachable Samples}
Using the set $\mathcal{X}_{R}$, we will formulate a DP trajectory optimization problem. To that end, we define a measure called the fraction of reachable samples (FRS) regarding output regions. The output space is partitioned into $n_o$ regions as $Y_{m}\coloneqq\left\{y\in \mathbb{R}^{n_y}: \left\| y - \widetilde{y}_{m} \right\|_{\infty} \leq \delta_{m} \right\} \subsetneq \mathbb{R}^{n_y}$, where $m \in \left[1, n_o \right]_d$ and $\widetilde{y}_{m} \in \mathbb{R}^{n_y}$ denotes the center of the regions $Y_{m}$. Then, we define the subset of $\mathcal{X}_{R}$ as follows:
\begin{equation}
	\mathcal{X}_{Y_{m}} \coloneqq \left\{ x \in \mathbb{R}^{n_x} : g(x)\in \mathrm{Int}\left(Y_{m}\right), \forall x \in \mathcal{X}_{R}  \right\} \textrm{.} 
	\label{X_Y_m}
\end{equation}
Using the set $\mathcal{X}_{Y_{m}}$, the FRS quantifies how many random samples are allocated in the region $Y_{m}$.   
\begin{definition}
Suppose that $m \in [0, n_{0}]_d$ is given and that $\mathcal{X}_{R}$ is not an empty set. The Fraction of Reachable Samples (FRS) is defined as the ratio: 
\begin{equation}
\mathfrak{F}_{m|N_s} \coloneqq \frac{ \mathrm{card}\left( \mathcal{X}_{Y_{m}}\right)}{\mathrm{card}\left(\mathcal{X}_{R} \right)} \textrm{.}
\end{equation}
where $N_s$ is $\mathrm{card}(\mathcal{X}_{R})$ and $\mathcal{X}_{Y_m}$ is defined in (\ref{X_Y_m}).
\label{def1}
\end{definition}
We need sufficient number of samples so that the FRS is a reliable property. In order to check the rate of change of the FRS with respect to the number of samples, we define its gradient as follows:
\begin{equation}
	\mathcal{G}\left(\mathfrak{F}_{m|N_s}, \Delta N_{s} \right) \coloneqq \frac{\mathfrak{F}_{m|({N_{s}+\Delta N_{s}})} - \mathfrak{F}_{m|{N_s}}}{\Delta N_{s}}
    \label{eq:grad_FRI}
\end{equation}
where $\Delta N_{s}$ denotes the number of new samples generated in $\mathcal{X}_{R}$. Using the gradient of FRS, we address the convergence of the FRS with respect to $N_s$ in the following theorem:
\begin{thm}
Let $\Delta N_s \in \mathbb{Z}^{++}$ and $m \in [0,n_o]_d$ be given and $\Delta N_s$ be less than $N_s$. Then, $\mathcal{G}\left(\mathfrak{F}_{m|N_s}, \Delta N_s \right)$ will converge to zero as $N_s \rightarrow \infty$, that is, 
\begin{equation}
	\lim_{N_{s} \rightarrow \infty} \mathcal{G}\left(\mathfrak{F}_{m|N_s}, \Delta N_s \right) = 0
\end{equation}
In other words, $\lim_{N_s \rightarrow \infty} \mathfrak{F}_{m|N_s} = \overline{\mathfrak{F}}_{m}$ where $\overline{\mathfrak{F}}_{m}$ represents a convergence value for the FRS.
\label{thm1}
\end{thm}
\begin{proof}
Suppose $\mathrm{card}(\mathcal{X}_{R}) = N_s$ and $\mathrm{card}(\mathcal{X}_{Y_m}) = N_m$. We consider a case in which all new samples are allocated within $Y_{m}$, which means the rate of change of the FRS is maximum. Hence, $|\mathcal{G}(\mathfrak{F}_{m|N_s},\Delta N_s)|$ is bounded such that
\begin{equation*}
\begin{split}
	&|\mathcal{G}\left( \mathfrak{F}_{m|N_s}, \Delta N_s\right)| \leq \left| \frac{N_{s} \mathfrak{F}_{m|N_{s}} + \Delta N_{s}}{\Delta N_{s} \left(  N_s + \Delta N_s \right)} - \frac{\mathfrak{F}_{m|N_{s}}}{\Delta N_s} \right|  \\
    &= \left| \frac{N_{s}\Delta N_{s} - N_{m} \Delta N_{s} }{N_{s} \Delta N_s \left( N_{s} + \Delta N_{s}  \right)} \right| = \left|\frac{N_{s} - N_{m} }{N_{s} \left( N_{s} + \Delta N_{s}  \right)} \right|
\end{split}
\end{equation*}
Since $\Delta N_s, N_m \ll N_{s}$,
\begin{equation*}
\begin{split}
 &0 \leq \lim_{N_s \rightarrow \infty} \left| \mathcal{G}\left( \mathfrak{F}_{m|N_s}, \Delta N_s\right) \right|  \leq \lim_{N_s \rightarrow \infty} \left|\frac{N_{s} - N_{m} }{N_{s} \left( N_{s} + \Delta N_{s}  \right)} \right| \\
 &=\lim_{N_s \rightarrow \infty} \left| \frac{ 1 - N_{m}/N_{s} }{ N_{s} + \Delta N_{s} } \right| = \lim_{N_s \rightarrow \infty} \left| \frac{1}{N_s} \right| = 0
\end{split}
\end{equation*}
It is therefore implied that $\lim_{N_s\rightarrow \infty} \mathcal{G}\left( \mathfrak{F}_{m|N_s}, \Delta N_s\right)=0$.
\end{proof}
For numerical implementation $N_s$ is determined by using the inequality $\max\left( G_{N_s} \right) \leq \varepsilon_{G}$, where $G(N_s) \coloneqq \left\{ n_{\mathcal{G}} \in \mathbb{R}^{+}: n_{\mathcal{G}} = | \mathcal{G}\left(\mathfrak{F}_{m|N_s}, \Delta N_s \right)|, m \in [1, n_{o}]_d  \right\}$ and $\varepsilon_{G}$ is a pre-defined small positive threshold.

Let us consider the set of output samples:
\begin{equation}
	\mathcal{Y}_{m} \coloneqq \left\{ y \in \mathbb{R}^{n_y}: y = g(x), \forall x \in \mathcal{X}_{Y_m}  \right\} 
\end{equation}
The mean vector and the covariance matrix of samples $y\in \mathcal{Y}_m$ are represented by $\mu_{y_{m}} \coloneqq \mathbb{E}[y]$ and $\mathbf{\Sigma}_{\mathcal{Y}_{m}} \coloneqq \mathbb{E} \left[ (y -\mu_{y_m}) (y-\mu_{y_m})^{\top} \right]$, respectively. To analyze the patterns of the output samples, the covariance matrix $\mathbf{\Sigma}_{\mathcal{Y}_m}$ is decomposed using the singular vale decomposition such that 
\begin{equation}
    \mathbf{\Sigma}_{\mathcal{Y}_{m}} = U_{m} \Omega_{m} V_{m}^{\top}
\end{equation}
where $U_{m}$ and $V_{m}$ are unitary matrices in $\mathbb{R}^{n_y \times n_y}$ and $\Omega_{m}$ represents a diagonal singular value matrix in $\mathbb{R}^{n_y \times n_y}$. We next define a vector that is associated with the largest singular value, to find the principle direction of the distribution of samples in $\mathcal{Y}_{m}$. 
\begin{definition}
Let $V_{m}=[\mathrm{col}_{1}\left(V_{m}\right),\dots,\mathrm{col}_{n_y}\left(V_{m}\right)]$ and $\Omega_{m} = \mathrm{diag}\left(\sigma_{m,1}, \dots, \sigma_{m,n_y}\right)$ where $\sigma_{m,i}$ denotes the singular values of  $\mathbf{\Sigma}_{\mathcal{Y}_{m}}$. 
The principal singular vector (PSV) of $\mathcal{Y}_{m}$ is defined as the singular vector corresponding to the largest singular value such that 
\begin{equation}
\mathfrak{P}_{m} \coloneqq \mathrm{col}_{j}\left( V_m \right), \quad \sigma_{max} = \sigma_{m,j}
\end{equation}
where $\mathfrak{P}_{m} \in \mathbb{R}^{n_y}$ and $\sigma_{max}$ denotes the maximum singular value of $\mathbf{\Sigma}_{\mathcal{Y}_m}$.
\end{definition}
The PSV $\mathfrak{P}_{m}$ is a critical property for constructing the transition dynamics of the DP process. 

\subsection{Dynamic Programming based on Sample Properties}
\label{sec:DP}
After computing the FRS, $\overline{\mathfrak{F}}_{m}$, and the PSV, $\mathfrak{P}_{m}$, with sufficient samples, we formulate the DP problem using a Markov Decision Process (MDP) to create an end-to-end trajectory. To start with, we define a discrete node associated with the output region $Y_{m}$ as follows:
\begin{equation}
	s_{m} \coloneqq \mathrm{node} \left( Y_{m} \right), m \in [1,n_o]_d \textrm{.}
\end{equation}
For DP, we represent the FRS and the PSV corresponding to the node index $s_m$ as $\mathfrak{F}_{s_m} = \overline{\mathfrak{F}}_{m}$ and $\mathfrak{P}_{s_m} = \mathfrak{P}_{m}$. In addition, we employ value iteration to solve the DP using the Bellman equation as follows:
\begin{equation*}
\begin{split}
&\mathfrak{D}^{\star} (s^{l}) = \max_{a} \left(  \mathfrak{R} \left( s^{l} \right) + \gamma \sum_{s^{l+1} \in S} \mathfrak{T}_{a} \left( s^{l},s^{l+1} \right)  \mathfrak{D}^{\star}\left( s^{l+1} \right) \right) 
\end{split}
\end{equation*}
where $s^{l}\in S$, $a$, $\mathfrak{R}$, $\mathfrak{T}_{a}$, $\gamma$, and $S$ are the node in the $l$-th iteration of the algorithm, the action, the reward, the transition dynamics, the discount factor, and the set of nodes, respectively. The reward function is defined to include as many samples as possible for the result of the DP:
\begin{equation}
\begin{split}
    \mathfrak{R}(s^{l}) = \left\{\begin{array}{ll} -\eta_{1} & \textrm{ if } \mathfrak{F}_{s^{l}} = 0 \\
    \eta_{2} + K_{\mathfrak{F}} \mathfrak{F}_{s^{l}} & \textrm{ if } s^{l} = s_{\phi} \\
    - \eta_{3} + K_{\mathfrak{F}} \mathfrak{F}_{s^{l}} & \textrm{ else } 
    \end{array}  \right.
\end{split}
\end{equation}
where $s_{\phi}$ denotes the node associated with the region containing the goal output, $\phi$. $K_{\mathfrak{F}} \in \mathbb{R}^{++}$ and $\eta_{1,2,3} \in \mathbb{R}^{++}$ are the gain for $\mathfrak{F}_{s^{l}}$ and the offset values for the reward functions, respectively.

If the PSV of $\mathbf{\Sigma}_{\mathcal{Y}_{m}}$ associated with the node $s_{m}$ is not well defined, e.g., when the distribution of samples is isometric or uniform, the transition dynamics of the DP process is considered deterministic. Otherwise, the transition dynamics is computed by the direction cosine between the PSV $\mathfrak{P}_{s^{l}}$ and an action vector $\pi(a) \in \mathbb{R}^{n_y}$ such that
\begin{equation}
\Pi\left(s^{l}, s^{l+1},a \right) \coloneqq \max\left(\left\{0, \frac{\mathfrak{P}_{s^{l}}^{\top}\pi(a)}{\left|\mathfrak{P}_{s^l} \right|\left| \pi(a) \right|}   \right\} \right) \textrm{.} 
\end{equation} 
Then, the transition dynamics is obtained by normalizing the direction cosine as follows:
\begin{equation}
    \mathfrak{T}_{a} \left(s^{l}, s^{l+1} \right) = \frac{\Pi \left( s^{l}, s^{l+1}, a \right)}{\sum_{\hat{s} \in \hat{S} }\Pi \left( s^{l}, \hat{s}, a \right)}
    \label{eq:transition_dyn}
\end{equation}
where $\hat{s}$ is an individual neighboring node of $s^{l}$, and $\hat{S}$ is the collection of all neighboring nodes, respectively. The reward and transition dynamics are designed to exclude infeasible output regions ($\mathfrak{F}_{s_m} =0$). 

Using DP, we obtain a sequential set of nodes $\mathcal{S}^{\star} \coloneqq \left\{ s^{1}, s^{2}, \cdots, s^{n_{dp}} \right\}$. The set is converted to $\mathcal{Q}^{\star} \coloneqq \left\{\mathrm{node}^{-1}(s^{1}), \mathrm{node}^{-1}(s^{2}), \cdots, \mathrm{node}^{-1}(s^{n_{dp}}) \right\}$ where $\mathrm{node}^{-1}(s_m)$ denotes the mapping of the node $s_m$ to its corresponding output region $Y_m$. This implies that the end-to-end trajectory generation problem can be formulated as a trajectory generation problems between the output regions in $\mathcal{Q}^{\star}$.

\section{TRAJECTORY GENERATION VIA REACHABILITY ANALYSIS}
\label{sec4}
Given the resulting sequence of output regions $\mathcal{Q}^{\star}$ from the DP, we now generate feasible trajectories between output regions connecting $\mathrm{node}^{-1}(s^{l})$ to $\mathrm{node}^{-1}(s^{l+1})$ for all $l\in [1,n_{dp}-1]_d$. After generating trajectories for all $l\in [1,n_{dp}-1]_d$, an entire end-to-end trajectory can be generated with feasibility guarantees.

\subsection{Reachability Analysis}
As discussed above, we seek to solve a nonlinear optimization problem for the sequence of output regions. The nonlinear optimization strategy requires a feasible initial condition so that the solution can converge to the local optimal point. For this reason, we compute a reachable set fulfilling the constraints given an initial state. The reachable set is defined for a continuous system with the contact force constraint as: 
\begin{definition}
Given an initial state $x_{0}\in \mathbb{R}^{n_x}$ and a time instance $t$, the reachable state set of the robotic system given a contact force constraint is defined as:
\begin{equation}
\begin{split}
	&\mathcal{R}_{t}^{x} \left(x_{0} \right) \coloneqq \{  z \in \mathbb{R}^{n_x} : z =x(t), \exists u([t_0, t]),  \\
    & \exists F_{c}([t_{0}, t]), C_{x} \left(  x(t) \right)\leq 0, C_{u}\left( u(t) \right) \leq 0,    \\
    & C_{x,u} \left(x(t), u(t) \right) \leq 0, \mathcal{D}\left( x(t) \right)F_c(t) \leq 0,  x(0) = x_0, \\
    & t \in [ t_{0}, t], \dot{x} = f_x(x) + f_u(x) u + f_c(x)F_c \}.
\end{split}\label{eq:gais}
\end{equation}
\end{definition}
This reachable state set can be extended over a time interval $[t_i,t_{i+1}]$ as follows:
\begin{equation}
\begin{split}
    \mathcal{R}^{x}_{[t_i, t_{i+1}]}\left( x_{0} \right) \coloneqq \bigcup_{t \in [t_i, t_{i+1}]} \mathcal{R}^{x}_{t} (x_{0}) \textrm{.}
\end{split}
\end{equation}
Using the reachable set for the states, the reachable set for the outputs is defined as:
\begin{equation}
	\mathcal{R}_{t}^{y}(x_0) \coloneqq \left\{\nu \in \mathbb{R}^{n_y}: \nu = g(\overline{x}), \forall \overline{x} \in \mathcal{R}_{t}^{x} (x_0) \right\} 
\label{eq:reach_x}
\end{equation}
which can be extended to the time interval $[t_i,t_{i+1}]$ as:
\begin{equation}
	\mathcal{R}_{[t_i, t_{i+1}]}^{y} (x_0) \coloneqq  \bigcup_{t \in [t_i, t_{i+1}]} \mathcal{R}_{t}^{y}(x_0) \textrm{.}
\end{equation}
Since this paper considers the discrete state space model coupled with a sampling-based approach, we approximate the reachable set, e.g. Eq.~(\ref{eq:gais}), with a discrete state space model.
 
Before computing the reachable sets, we check the state bounds using the discrete state space model (\ref{eq:time_discrete}): 
\begin{equation}
\begin{split}
 	&\left\|  x_{k+1} - x_{k} \right\| = T \left\|   B_{1}  + \frac{T}{2}  B_{2}  +  \frac{\mathcal{O} \left( T^{2} \right)}{T} \right\| \\
 	&\leq T \left\| \left( I + \mathcal{Z}_{1}\right)\right\| \left\| B_{1}  \right\| + K|T| = \mathcal{Z}_{2}(T,x_{k})
 \end{split}
 \label{eq:inequality01}
\end{equation}
where $\mathcal{Z}_{1} \coloneqq J_x(x_k) + J_u(x_k)u_{k} + J_{\overline{c}}(x_k)F_{c,k}$, $J_{x}(x_{k}) = \frac{\partial f_{x}}{\partial x}(x_{k})$, $J_{u}(x_{k}) = \frac{\partial f_{u}}{\partial x}(x_k)$, and $J_{\overline{c}}(x_k) = \frac{\partial f_c}{\partial x}(x_{k})$. $T$ is the time increment and it should be small satisfying $\mathcal{O}\left( T^{2} \right) < K|T|$. Also, the norm of the output update is bounded by: 
\begin{equation}
\begin{split}
	\left\| y_{k+1} - y_{k} \right\| &= \left\| J_{y}(x_{k}) \left( x_{k+1} - x_{k} \right) \right\| \\
    &\leq \left\| J_{y}(x_{k}) \right\| \left\| x_{k+1} - x_{k} \right\| = \mathcal{Z}_{3}(T,x_{k})
\end{split}
\label{ineq:output}
\end{equation}
where $J_{y}(x_{k})=\frac{\partial g}{\partial x}(x_{k})$. Based on (\ref{ineq:output}), we define the closed ball in the output space as follows:
\begin{equation}
	\mathcal{B}^{y}\left(T, x_{0} \right) \coloneqq \left\{ y \in \mathbb{R}^{n_y} : \left\| y - g(x_0) \right\| \leq \mathcal{Z}_{3}(T,x_0)  \right\} \textrm{.}
	\label{eq:norm_ball}
\end{equation}
Since the reachable output set is a subset of $\mathcal{B}^{y}(T,x_0)$, it is necessary to consider a time interval wider than $[0,T_{\backslash min}]$ in the reachability analysis, where $T_{\backslash min} \coloneqq \min(\{ t : t=k\Delta t, \phi \in \mathcal{B}^{y}(t,x_0),k \in \mathbb{Z}^{+} \})$. 

The reachable set is numerically constructed using the discrete state space model. To start, we formulate the optimization problem with the state, $x_{k}$, and input, $u$:
\begin{subequations}\label{opt3}
\begin{align}
	\min_{F_{c},x_{k+1}} &\quad F_{c}^{\top} W_{c} F_{c} \\
    \textrm{s.t.} &\quad  x_{k+1} = \mathcal{F}\left(x_{k},u,F_{c}\right), \mathcal{D}_{c}(x_k)F_{c} \leq 0,\\
    &\quad C_{x,u}(x_{k},u) \leq 0, C_{x}(x_{k+1}) \leq 0 \textrm{.}
\end{align}
\end{subequations}
Let us consider an initial state, $x_{k} = x_0$. We draw random input samples from a Gaussian distribution $u \sim \mathcal{N}(\mu_u, \mathbf{\Sigma}_{u})$ where $\mu_{u} \coloneqq \mathbb{E}(u)$ and $\mathbf{\Sigma}_{u} \coloneqq \mathbb{E}[(u - \mu_{u})(u -\mu_{u})^{\top}]$ denote the mean vector and the covariance matrix of the input samples, respectively. For all of the generated input samples, we only select samples that fulfill the input constraints.

Via our numerical strategy, the reachable state set at $\Delta t$ is approximated as $\overline{\mathcal{R}}_{\Delta t}^{x} (x_0)$ in the form of the collection $x_{1}$ from the optimization (\ref{opt3}) for all input samples. The reachable set of the output samples is computed as: 
\begin{equation}
\overline{\mathcal{R}}^{y}_{\Delta t} \coloneqq \{ \nu \in \mathbb{R}^{n_y} : \nu = g(\overline{x}), \overline{x} \in \overline{\mathcal{R}}^{x}_{\Delta t}(x_0) \}  \textrm{.}
\end{equation}
We can extend this method to obtain the reachable set over multiple time steps. Suppose a time step, $T_{k} = k \Delta t$, and a reachable set for the previous time step, $\overline{\mathcal{R}}^{x}_{0,T_{k-1}}(x_0)$. Our strategy is to solve the optimization problem in (\ref{opt3}) with respect to the samples in the set $\{\overline{x} : g(\overline{x}) \in \mathrm{bd}(\overline{\mathcal{R}}_{[0,T_{k-1}]}^{y}(x_0))\}$ instead of with respect to $x_{k}$. Then the computation of the reachable set at $T_{k}$, $\overline{\mathcal{R}}_{T_{k}}^{x}(x_0)$ results in a collection of $x_{k+1}$ for all $\overline{x}$ and $u$.
Based on $\overline{\mathcal{R}}_{[0,T_{k-1}]}^{x}(x_0)$ and $\overline{\mathcal{R}}_{T_{k}}^{x}(x_0)$, we can obtain the reachable state set over the time horizon $[0,T_{k}]$ and the corresponding output set as:
\begin{subequations}
\begin{align}
&\overline{\mathcal{R}}^{x}_{[0,T_{k}]}(x_{0}) = \overline{\mathcal{R}}^{x}_{[0,T_{k-1}]}(x_{0}) \bigcup\overline{\mathcal{R}}^{x}_{T_{k}}(x_{0})\\
&\overline{\mathcal{R}}^{y}_{[0,T_{k}]}(x_0) = \{ \nu \in \mathbb{R}^{n_y} : \nu = g(\overline{x}), \overline{x} \in \overline{\mathcal{R}}^{x}_{[0,T_{k}]}(x_0) \}
\label{reach_y_time}
\end{align}
\end{subequations}
where $k \geq 2$, $\overline{\mathcal{R}}^{x}_{[0,T_{1}]}(x_0) = \overline{\mathcal{R}}^{x}_{\Delta t}(x_0)$, and $\overline{\mathcal{R}}^{y}_{[0,T_{1}]}(x_0) = \overline{\mathcal{R}}^{y}_{\Delta t}(x_0)$. In terms of the computational efficiency, the proposed strategy can reduce the computational cost to $\mathcal{O}(N_{b}^{N_{u}})$, where $N_b$ and $N_u$ denote the number of samples at the boundary of the reachable set and the number of input samples, respectively.   

As mentioned above, we must check whether the desired goal belongs to the reachable set of output samples before solving the optimal control problem given an initial state. We first approximate $\overline{\mathcal{R}}^{x}_{[0,T]}(x_0)$ by a non-convex hull. We then check whether the desired goal output $y_{\phi}$ is reachable or not: $	y_{\phi} \in \mathrm{Nconv}\left(\overline{\mathcal{R}}^{y}_{[0,T]}(x_0)\right)$, where $\overline{\mathcal{R}}^{y}_{[0,T]}(x_0)$ is addressed in  (\ref{reach_y_time}).
If we find a $T$ satisfying $y_{\phi} \in \mathrm{Nconv}\left(\overline{\mathcal{R}}^{y}_{[0,T]}(x_0)\right)$, then we can conclude that it implies that the desired goal output position $y_{\phi}$ is achievable over the time interval $[0,t]$ where $t\geq T \geq T_{\backslash min}$. 

\subsection{Optimal Control}
\label{sec_optcont}
To generate the whole trajectory, we recursively solve the optimal control problem between $\mathrm{node}^{-1}(s^{l})$ and $\mathrm{node}^{-1}(s^{l+1})$ for all $l\in [1,n_{dp}-1]_d$ in $\mathcal{Q}^{\star}$, as defined in Section \ref{sec3}. To begin the optimal control process, the time interval $T^{l}$ and the desired output $\widetilde{y}^{l+1}$ for the $l$-th optimal control problem are determined through the previous reachability analysis and initial state, $x_0^{l}$. We also define a performance measure of the optimal control problem in the discrete time domain as follows: 
\begin{equation*}
\begin{split}
&\mathcal{L}^{l}(u(.),N^{l} ) \coloneqq \sum_{k=0}^{N^{l}-1} \left(u_{k}^{\top} W_{1} u_{k} + F_{c,k}^{\top}W_{2} F_{c,k}\right) + L^{l}\left( \xi_N \right) \\
&L^{l}\left( \xi_{N} \right) \coloneqq \left(\widetilde{y}^{l+1} - \xi_{N} \right)^{\top} W_{3} \left( \widetilde{y}^{l+1} -\xi_{N} \right)
\end{split}
\end{equation*} 
where $\xi$ denotes the trajectory of the output and $N^{l} = T^{l}/\Delta t$. $W_{1}\in \mathbb{S}_{n_u}^{++}$, $W_{2} \in \mathbb{S}_{n_c}^{++}$, and $W_3 \in \mathbb{S}_{n_y}^{++}$ are the weighting matrices for the components of the performance measure. 

The optimal control problem is formulated using the reachable sets, the constraints, and the discrete state space model as follows:
\begin{subequations}\label{opt2}
\begin{align}
    \min_{\zeta(.), u(.)}& \quad \mathcal{L}^{l}(u(.), N^{l}) \\
    \textrm{s.t.}& \quad \xi_{k} \in \overline{\mathcal{R}}^{y}_{[0,T^{l}]}(x_0^{l}) \\
    & \quad \zeta_{k} \in  \overline{\mathcal{R}}^{x}_{[0,T^{l}]}(x_0^{l}) \\
    & \quad \zeta_{k+1} = \mathcal{F}\left(\zeta_{k}, u_{k}, F_{c,k}  \right) \\
    & \quad  C_{x} \left( \zeta_{k}\right) \leq 0, C_{u}\left( u_{k} \right) \leq 0,  \\
    & \quad C_{x,u}\left(\zeta_{k}, u_{k} \right) \leq 0,\mathcal{D}(\zeta_{k})F_{c,k} \leq 0, \\
    & \quad \zeta_{0} =x_{0}^{l}, \xi_{0} = g(x_{0}^{l})
\end{align} 
\end{subequations}
where $\zeta$ denotes the state trajectory. As a result of the optimal control problem, we obtain the state trajectory to control the system output from region $\mathrm{node}^{-1}(s^{l})$ to region $\mathrm{node}^{-1}(s^{l+1})$ in $\mathcal{Q}^{\star}$:
\begin{equation}
\Psi_{l} \coloneqq  \mathrm{Vertcat}\left( \zeta_{k}^{\top} : \forall k \in [0, N^{l}]_d  \right) \textrm{.}
\end{equation}
This process is sequentially implemented for all pairs $(\mathrm{node}^{-1}(s^{l}),\mathrm{node}^{-1}(s^{l+1}))$ in $\mathcal{Q}^{\star}$ by replacing the initial state $x_0^{l}$ with the last component of $\Psi_{l-1}$ when $l\neq 1$. In particular, we set $x_{0}^{1} = x_0$ and $\widetilde{y}^{n_{dp}} = \phi$. Note that all initial states in the optimal control problem are feasible due to the constraints (34b) and (34c). After solving for all trajectories $\Psi_{1}, \dots, \Psi_{n_{dp} -1}$, we obtain the start-to-end state trajectory by connecting the individual trajectories such that: 
\begin{equation}
	\psi = \mathrm{Vertcat} \left(\Psi_{1}, \dots, \Psi_{n_{dp} -1}\right) \textrm{.}
\end{equation}
The full state trajectory and the corresponding input are able to control the robotic system, fulfilling not only the required differential constraints but also the contact force constraints.  
\begin{figure*}
\centering
\begin{minipage}[t]{0.28\linewidth}
\includegraphics[width=\linewidth, valign=c]{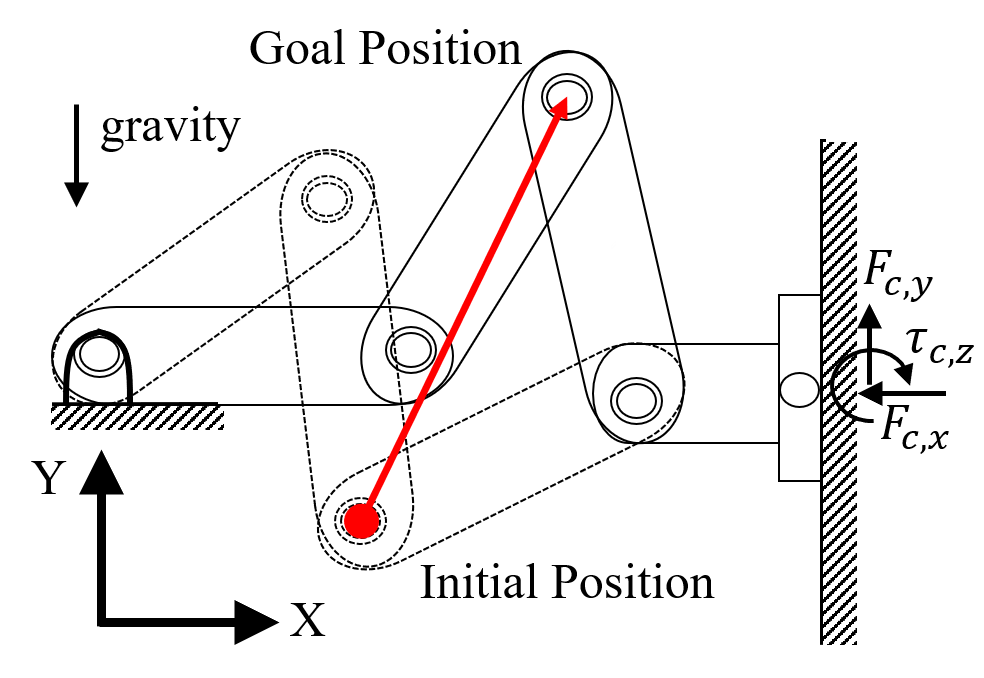}
\centerline{\footnotesize{(a)}}
\end{minipage}
\begin{minipage}[t]{0.39\linewidth}
\includegraphics[width=\linewidth,valign=c]{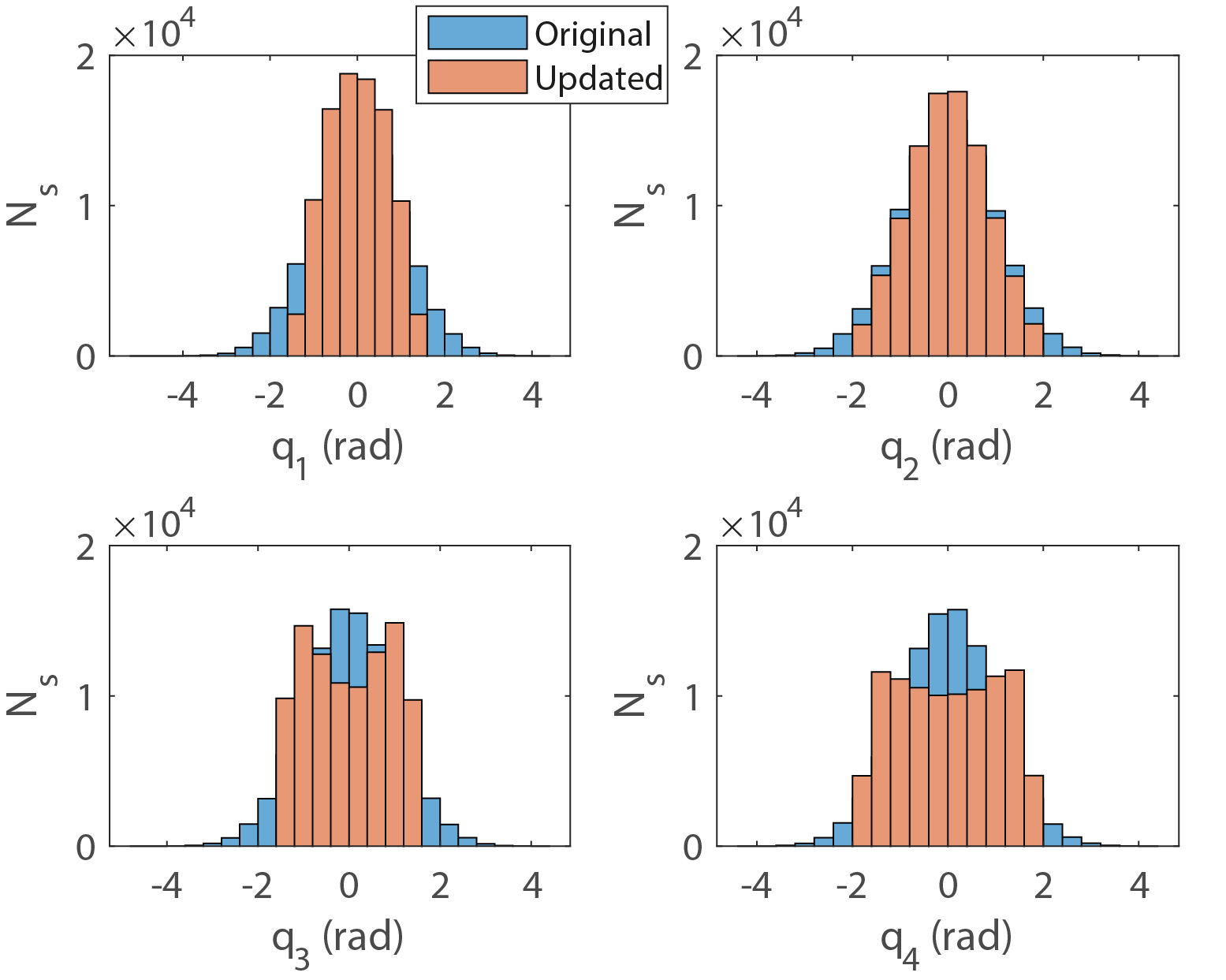}
\centerline{\footnotesize{(b)}}
\end{minipage}
\begin{minipage}[t]{0.26\linewidth}
\includegraphics[width=\linewidth,valign=c]{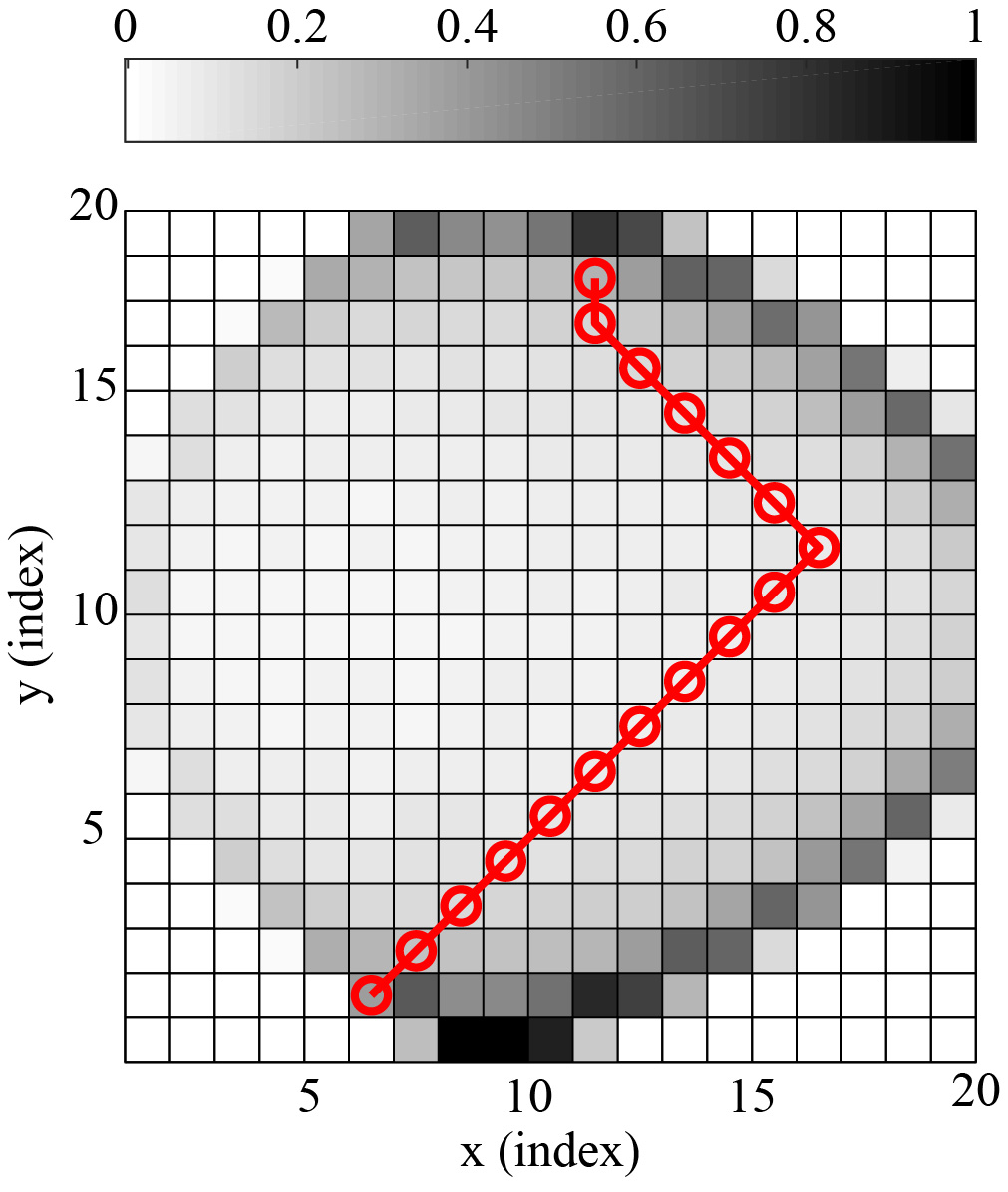}
\centerline{\footnotesize{(c)}}
\end{minipage}
\caption{(a) Conceptual robotic system, (b) histograms of the original samples and the samples updated to fulfill the constraints, including contact force constraints. $N_s$ is the number of samples in desired specific range, (c) is the color-map of the FRS $\mathfrak{F}_m$ and the solution of the first DP.}
\label{Fig1}
\end{figure*}

\begin{figure*}
\centering
\begin{minipage}[h]{0.32\linewidth}
\includegraphics[width=\linewidth,valign=c]{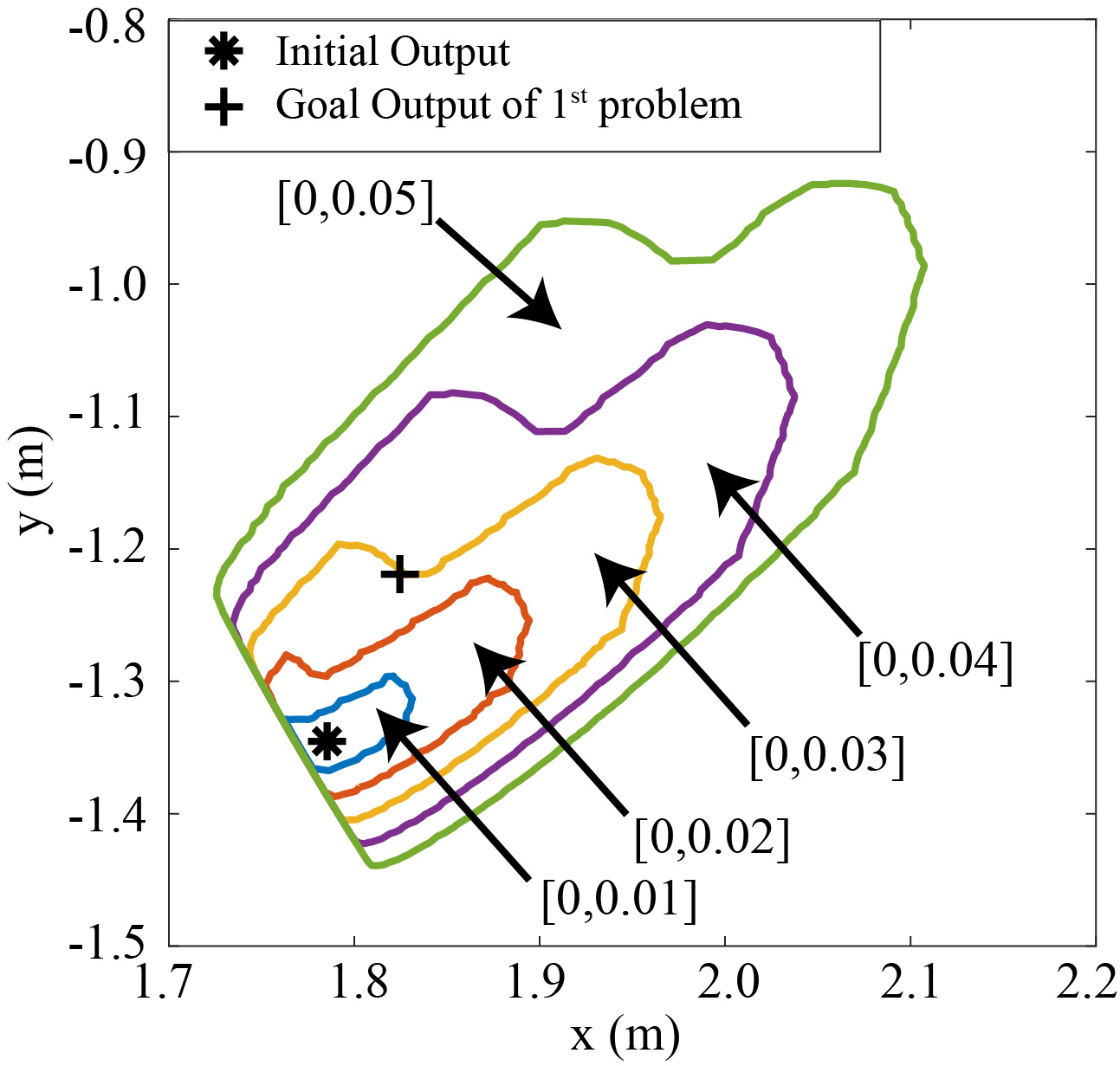}
\centerline{\footnotesize{(a)}}
\end{minipage}
\begin{minipage}[h]{0.32\linewidth}
\includegraphics[width=\linewidth,valign=c]{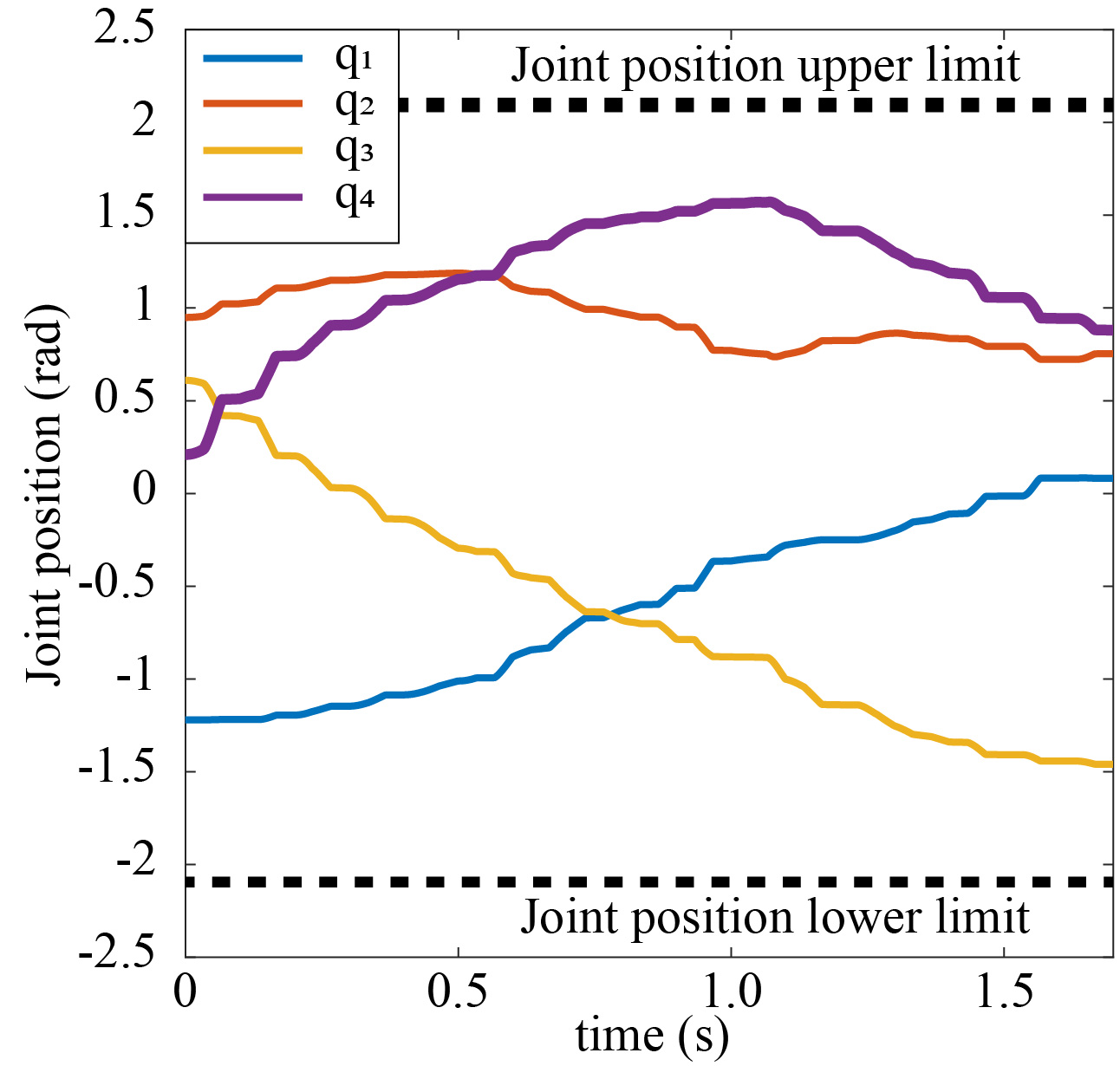}
\centerline{\footnotesize{(b)}}
\end{minipage}
\begin{minipage}[h]{0.32\linewidth}
\includegraphics[width=\linewidth,valign=c]{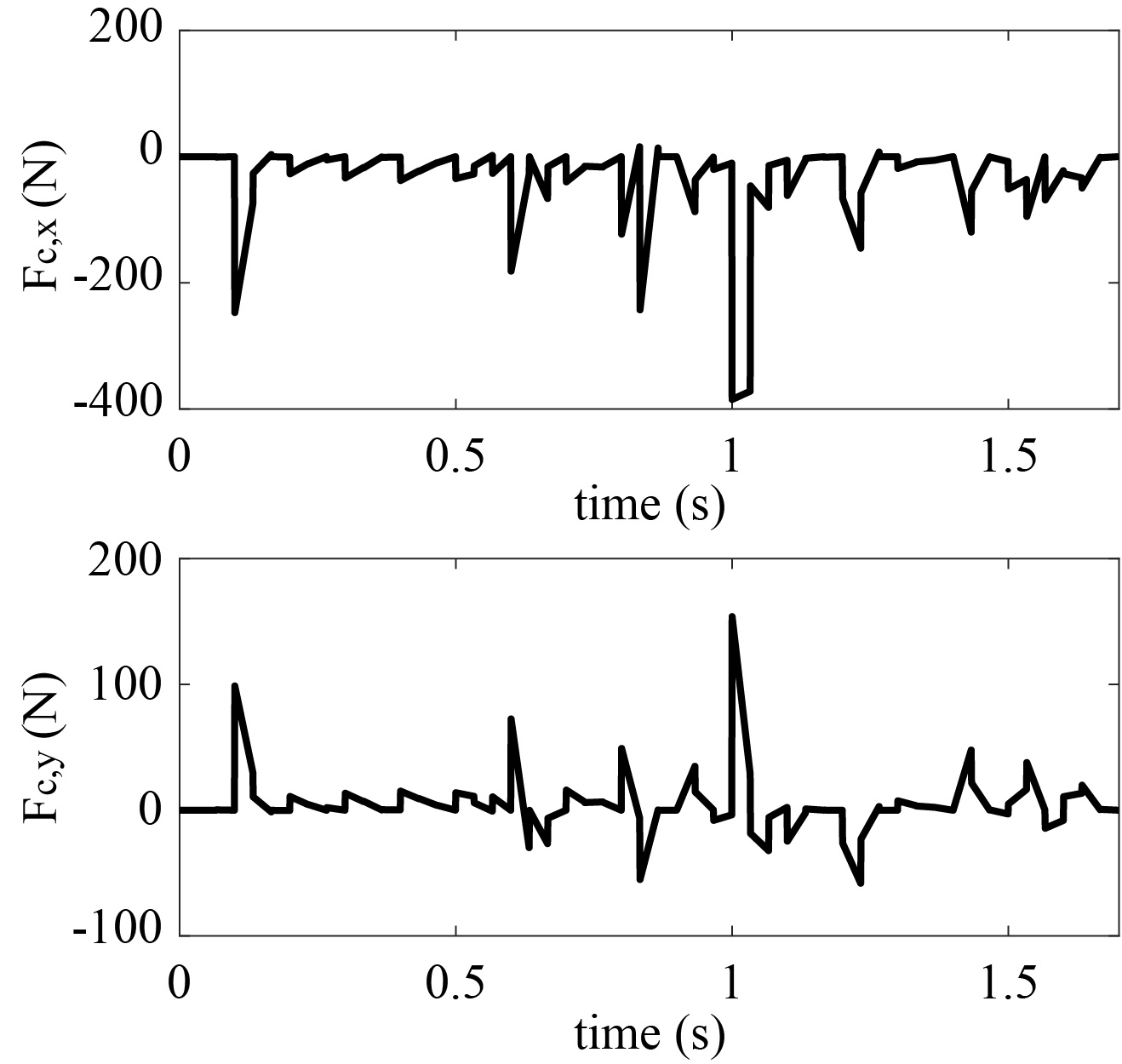}
\centerline{\footnotesize{(c)}}
\end{minipage}
\caption{(a) Non-convex hull of the reachable output set within the time interval: $\mathrm{Nconv}\left(\overline{\mathcal{R}}^{y}_{[0,T]}(x_0) \right)$, (b) joint position trajectory, (c) contact force}
\label{Fig2}
\end{figure*}

\section{SIMULATIONS}
\label{sec5}
In this simulation, we consider a 2-dimensional output space and a 4-DOF planar robot with a contact point at the end-effector as shown in Fig. \ref{Fig1}(a). For the simulation, we use the optimization toolbox of MATLAB and the tool DYNOPT \cite{cizniar2006matlab} on a laptop with an i7-8650U CPU and $16.0$ GB RAM.

\subsection{A Planar Robot with Contact}
A 4-DOF planar robot is set with parameters $M = [ 2.0, 1.5, 1.0, 1.0]$ kg and $L = [1.0, 1.5, 2.5, 1.0]$ m where $M$ and $L$ denote the masses and link lengths, respectively. It is assumed that each link's center of mass is located at the geometric center of the link. We consider joint position, velocity, and torque limits such that:
\begin{equation*}
\begin{split}
	-(2/3)\pi \leq q_{i} \leq (2/3)\pi \quad \forall i \in [1,4]_d \\   
	-(3/2)\pi \leq \dot{q}_{i} \leq (3/2)\pi \quad \forall i \in [1,4]_d \\
    -1000 \leq u_{i} \leq 1000 \quad \forall i \in [1, 4]_d\end{split}
\end{equation*}
where $q_{i}$, $\dot{q}_i$, and $u_{i}$ denote the $i$-th joint position, velocity, and torque input, respectively. We also consider contact geometry constraints for the end-effector. Specifically, the contact position and the orientation of the end-effector should be consistent, and its velocity should be zero. The contact force constraint is defined as follows:
\begin{equation*}
\begin{split}
    &-\mu |F_{c,x}| < F_{c,y} <\mu |F_{c,x}|, \quad F_{c,x} <0 \\
    &- L_c |F_{c,x}| \leq \tau_{c,z} \leq L_c |F_{c,x}| 
\end{split}
\end{equation*}
where $F_c = \left[F_{c,x}, F_{c,y}, \tau_{c,z} \right]^{\top} $ and $L_c$ denotes the moment arm of the contact link. The focus of the performance for this simulation is on the $2$nd link which is required to move while a solid contact should be maintained on the end-effector, as shown in Fig. \ref{Fig1}(a). The initial configuration, the joint velocity, and the goal output are $q_{0} = [-1.22, 0.949, 0.610, 0.210]$ rad, $\dot{q}_0 = [0.0, 0.0, 0.0]$ rad/s, and $\phi = [2.0, 1.2]$ m, respectively.

We then generate random samples. The threshold $\varepsilon_{G}$ used to determine the appropriate number of samples is set to be $6.0 \times 10^{-5}$. The Jacobian matrices of all the constraint functions are full row rank. As a result, the required number of samples is $2.0 \times 10^{5}$ to obtain the reliable FRS. The raw random samples and the modified samples, which fulfill the constraints, are represented by histograms as shown in Fig. \ref{Fig1}(b). Specifically, the updated samples satisfy not only joint position limits but also the contact geometric constraints. We next consider a $20 \times 20$ grid box to create trajectories using the DP process described in Section \ref{sec:DP}. As a result of the DP process, the trajectory generation problem over long-term interval is broken into $16$ short-term problems as shown Fig. \ref{Fig1}(c).

The reachable sets are numerically computed to check whether or not the goal outputs resulting from solving the DP problem are reachable via the proposed approach. Fig. \ref{Fig2}(a) describes the non-convex hull of the reachable output sets with respect to small time intervals $[0,0.01]$, $[0,0.02]$, $[0, 0.03]$, $[0, 0.04]$, and $[0, 0.05]$. The non-convex hull of the reachable output set, $\mathrm{Nconv}\left(\overline{\mathcal{R}}^{y}_{[0,0.03]}(x_0) \right) $, contains both the initial and goal output positions from the first trajectory generation problem. Therefore, we set the final time interval of the first nonlinear optimization as $0.1$ s $>$ $0.03$ s based on the reachable set of Fig. \ref{Fig2}(a). We repeat the process for the remaining problems. Fig. \ref{Fig2}(b) shows the generated joint position trajectory, and Fig. \ref{Fig2}(c) shows the corresponding contact force. The final configuration of the robot is $q_{f} = [0.084, 0.755, -1.460, 0.880]$ rad and the final output is $g(q_{f}) = [2.0, 1.2]$ m, which is identical to the desired goal output. Furthermore, all considered constraints, related to the joint position, joint velocity, and input limits and the contact geometry/force constraints, are fulfilled while achieving the desired goal.

\section{CONCLUSION AND FUTURE WORK}
This paper proposes an approach to generate feasible trajectories for robotic systems with contact force constraints. The proposed approach consists of a sampling-based method and two optimization processes to generate trajectories that maintain solid contacts in an effective way. Using properties from our sampling approach, the end-to-end trajectory generation problem over a long-term interval is replaced by multiple sub-problems over short-term intervals. This strategy also enables us to perform numerical reachability analysis for the finite time interval before implementing an optimal control process. The simulation results show that the proposed approach successfully generates a feasible trajectory with contact force constraints. 

In the near future, we will conduct an extended analysis of our method. Real experiments on a robot will be pursued to verify the scalability of our method. We will also apply our method to more complex systems such as dual-arm and bipedal robots.        





\section*{ACKNOWLEDGMENTS}
The authors would like to thank the members of the Human Centered Robotics Laboratory at The University of Texas at Austin for their great help and support. This work was supported by an NSF Grant\# 1724360 and partially supported by an ONR Grant\# N000141512507.

\bibliographystyle{IEEEtran}
\bibliography{autosam}

\begin{thebibliography}{10}
\providecommand{\url}[1]{#1}
\csname url@samestyle\endcsname
\providecommand{\newblock}{\relax}
\providecommand{\bibinfo}[2]{#2}
\providecommand{\BIBentrySTDinterwordspacing}{\spaceskip=0pt\relax}
\providecommand{\BIBentryALTinterwordstretchfactor}{4}
\providecommand{\BIBentryALTinterwordspacing}{\spaceskip=\fontdimen2\font plus
\BIBentryALTinterwordstretchfactor\fontdimen3\font minus
  \fontdimen4\font\relax}
\providecommand{\BIBforeignlanguage}[2]{{%
\expandafter\ifx\csname l@#1\endcsname\relax
\typeout{** WARNING: IEEEtran.bst: No hyphenation pattern has been}%
\typeout{** loaded for the language `#1'. Using the pattern for}%
\typeout{** the default language instead.}%
\else
\language=\csname l@#1\endcsname
\fi
#2}}
\providecommand{\BIBdecl}{\relax}
\BIBdecl

\bibitem{sentis2005synthesis}
L.~Sentis and O.~Khatib, ``Synthesis of whole-body behaviors through
  hierarchical control of behavioral primitives,'' \emph{International Journal
  of Humanoid Robotics}, vol.~2, no.~04, pp. 505--518, 2005.

\bibitem{mistry2012operational}
M.~Mistry, ``Operational space control of constrained and underactuated
  systems,'' \emph{Robotics: Science and systems VII}, pp. 225--232, 2012.

\bibitem{escande2014hierarchical}
A.~Escande, N.~Mansard, and P.-B. Wieber, ``Hierarchical quadratic programming:
  Fast online humanoid-robot motion generation,'' \emph{The International
  Journal of Robotics Research}, vol.~33, no.~7, pp. 1006--1028, 2014.

\bibitem{stephens2010dynamic}
B.~J. Stephens and C.~G. Atkeson, ``Dynamic balance force control for compliant
  humanoid robots,'' in \emph{IEEE/RSJ International Conference on Intelligent
  Robots and Systems}, 2010, pp. 1248--1255.

\bibitem{jia2002min}
D.~Jia and B.~Krogh, ``Min-max feedback model predictive control for
  distributed control with communication,'' in \emph{Proceedings of American
  Control Conference}, vol.~6, 2002, pp. 4507--4512.

\bibitem{bravo2006robust}
J.~M. Bravo, T.~Alamo, and E.~F. Camacho, ``Robust {MPC} of constrained
  discrete-time nonlinear systems based on approximated reachable sets,''
  \emph{Automatica}, vol.~42, no.~10, pp. 1745--1751, 2006.

\bibitem{gonzalez2011online}
R.~Gonzalez, M.~Fiacchini, T.~Alamo, J.~L. Guzm{\'a}n, and F.~Rodr{\'\i}guez,
  ``Online robust tube-based {MPC} for time-varying systems: a practical
  approach,'' \emph{International Journal of Control}, vol.~84, no.~6, pp.
  1157--1170, 2011.

\bibitem{rakovic2012parameterized}
S.~V. Rakovic, B.~Kouvaritakis, M.~Cannon, C.~Panos, and R.~Findeisen,
  ``Parameterized tube model predictive control,'' \emph{IEEE Transactions on
  Automatic Control}, vol.~57, no.~11, pp. 2746--2761, 2012.

\bibitem{subramanian2017novel}
S.~Subramanian, S.~Lucia, and S.~Engell, ``A novel tube-based output feedback
  {MPC} for constrained linear systems,'' in \emph{Proceedings of American
  Control Conference}, 2017, pp. 3060--3065.

\bibitem{mitchell2001validating}
I.~Mitchell, A.~M. Bayen, and C.~J. Tomlin, ``Validating a {Hamilton-Jacobi}
  approximation to hybrid system reachable sets,'' in \emph{International
  Workshop on Hybrid Systems: Computation and Control}.\hskip 1em plus 0.5em
  minus 0.4em\relax Springer, 2001, pp. 418--432.

\bibitem{tomlin2000game}
C.~J. Tomlin, J.~Lygeros, and S.~S. Sastry, ``A game theoretic approach to
  controller design for hybrid systems,'' \emph{Proceedings of the IEEE},
  vol.~88, no.~7, pp. 949--970, 2000.

\bibitem{habets2006reachability}
L.~Habets, P.~J. Collins, and J.~H. van Schuppen, ``Reachability and control
  synthesis for piecewise-affine hybrid systems on simplices,'' \emph{IEEE
  Transactions on Automatic Control}, vol.~51, no.~6, pp. 938--948, 2006.

\bibitem{mitchell2005time}
I.~M. Mitchell, A.~M. Bayen, and C.~J. Tomlin, ``A time-dependent
  {Hamilton-Jacobi} formulation of reachable sets for continuous dynamic
  games,'' \emph{IEEE Transactions on automatic control}, vol.~50, no.~7, pp.
  947--957, 2005.

\bibitem{maiga2016comprehensive}
M.~Maiga, N.~Ramdani, L.~Trav{\'e}-Massuy{\`e}s, and C.~Combastel, ``A
  comprehensive method for reachability analysis of uncertain nonlinear hybrid
  systems,'' \emph{IEEE Transactions on Automatic Control}, vol.~61, no.~9, pp.
  2341--2356, 2016.

\bibitem{summers2013stochastic}
S.~Summers, M.~Kamgarpour, C.~Tomlin, and J.~Lygeros, ``Stochastic system
  controller synthesis for reachability specifications encoded by random
  sets,'' \emph{Automatica}, vol.~49, no.~9, pp. 2906--2910, 2013.

\bibitem{lesser2014reachability}
K.~Lesser and M.~Oishi, ``Reachability for partially observable discrete time
  stochastic hybrid systems,'' \emph{Automatica}, vol.~50, no.~8, pp.
  1989--1998, 2014.

\bibitem{kurzhanski2001dynamic}
A.~B. Kurzhanski and P.~Varaiya, ``Dynamic optimization for reachability
  problems,'' \emph{Journal of Optimization Theory and Applications}, vol. 108,
  no.~2, pp. 227--251, 2001.

\bibitem{asarin2000approximate}
E.~Asarin, O.~Bournez, T.~Dang, and O.~Maler, ``Approximate reachability
  analysis of piecewise-linear dynamical systems,'' in \emph{International
  Workshop on Hybrid Systems: Computation and Control}.\hskip 1em plus 0.5em
  minus 0.4em\relax Springer, 2000, pp. 20--31.

\bibitem{kariotoglou2013approximate}
N.~Kariotoglou, S.~Summers, T.~Summers, M.~Kamgarpour, and J.~Lygeros,
  ``Approximate dynamic programming for stochastic reachability,'' in
  \emph{Proceedings of European Control Conference}, 2013, pp. 584--589.

\bibitem{maidens2015reachability}
J.~Maidens and M.~Arcak, ``Reachability analysis of nonlinear systems using
  matrix measures,'' \emph{IEEE Transactions on Automatic Control}, vol.~60,
  no.~1, pp. 265--270, 2015.

\bibitem{arcak2017simulation}
M.~Arcak and J.~Maidens, ``Simulation-based reachability analysis for nonlinear
  systems using componentwise contraction properties,'' \emph{arXiv preprint
  arXiv:1709.06661}, 2017.

\bibitem{yang2017efficient}
Y.~Yang, W.~Merkt, H.~Ferrolho, V.~Ivan, and S.~Vijayakumar, ``Efficient
  humanoid motion planning on uneven terrain using paired forward-inverse
  dynamic reachability maps,'' \emph{IEEE Robotics and Automation Letters},
  vol.~2, no.~4, pp. 2279--2286, 2017.

\bibitem{guan2006reachable}
Y.~Guan and K.~Yokoi, ``Reachable space generation of a humanoid robot using
  the monte carlo method,'' in \emph{IEEE/RSJ International Conference on
  Intelligent Robots and Systems}, 2006, pp. 1984--1989.

\bibitem{shkolnik2009reachability}
A.~Shkolnik, M.~Walter, and R.~Tedrake, ``Reachability-guided sampling for
  planning under differential constraints,'' in \emph{IEEE International
  Conference on Robotics and Automation}, 2009, pp. 2859--2865.

\bibitem{boor1999gaussian}
V.~Boor, M.~H. Overmars, and A.~F. Van Der~Stappen, ``The {Gaussian} sampling
  strategy for probabilistic roadmap planners,'' in \emph{IEEE International
  Conference on Robotics and Automation}, vol.~2, 1999, pp. 1018--1023.

\bibitem{patil2012estimating}
S.~Patil, J.~Van Den~Berg, and R.~Alterovitz, ``Estimating probability of
  collision for safe motion planning under gaussian motion and sensing
  uncertainty,'' in \emph{IEEE International Conference on Robotics and
  Automation}, 2012, pp. 3238--3244.

\bibitem{carpentier2017multi}
J.~Carpentier and N.~Mansard, ``Multi-contact locomotion of legged robots,''
  \emph{IEEE Transactions on Robotics}, 2018.

\bibitem{hauser2008motion}
K.~Hauser, T.~Bretl, J.-C. Latombe, K.~Harada, and B.~Wilcox, ``Motion planning
  for legged robots on varied terrain,'' \emph{The International Journal of
  Robotics Research}, vol.~27, no. 11-12, pp. 1325--1349, 2008.

\bibitem{stilman2007task}
M.~Stilman, ``Task constrained motion planning in robot joint space,'' in
  \emph{IEEE/RSJ International Conference on Intelligent Robots and Systems},
  2007, pp. 3074--3081.

\bibitem{caron2015stability}
S.~Caron, Q.-C. Pham, and Y.~Nakamura, ``Stability of surface contacts for
  humanoid robots: Closed-form formulae of the contact wrench cone for
  rectangular support areas,'' in \emph{IEEE International Conference on
  Robotics and Automation}, 2015, pp. 5107--5112.

\bibitem{cizniar2006matlab}
M.~Cizniar, M.~Fikar, M.~Latifi \emph{et~al.}, ``A matlab package for dynamic
  optimisation of processes,'' in \emph{Proceedings of International
  Scientific-Technical Conference-PROCESS CONTROL}.\hskip 1em plus 0.5em minus
  0.4em\relax Kouty nad Desnou, Czech Republic, 2006.

\end{thebibliography}

\end{document}